\documentclass{article} 
\usepackage{iclr2027_conference,times}

\IfFileExists{iclr/sections/00_abstract.tex}
  {\def\paperroot{iclr/}}{\def\paperroot{}}
\makeatletter
\edef\input@path{{\paperroot styles/}}
\makeatother

\usepackage{algorithm}
\usepackage{algorithmic}

\usepackage{amsmath,amsfonts,bm}

\def\eqref#1{equation~\ref{#1}}

\def\1{\bm{1}}

\DeclareMathAlphabet{\mathsfit}{\encodingdefault}{\sfdefault}{m}{sl}
\SetMathAlphabet{\mathsfit}{bold}{\encodingdefault}{\sfdefault}{bx}{n}

\newcommand{\R}{\mathbb{R}}

\newcommand{\softmax}{\mathrm{softmax}}

\usepackage{hyperref}
\usepackage{url}

\usepackage[T1]{fontenc}
\usepackage{microtype}
\usepackage{graphicx}
\usepackage{amsmath,amssymb,amsthm,mathtools,bm}
\usepackage{booktabs,tabularx,array,multirow}
\usepackage{longtable}
\usepackage{enumitem}
\usepackage{needspace}
\usepackage{tikz}
\usetikzlibrary{arrows.meta,positioning,fit}
\hypersetup{
  hidelinks,
  pdftitle={Change the Product, Keep the Parameters: Associative Algebra Layers for Transformers},
  pdfsubject={Associative low-rank algebra products for Transformer layers}
}
\usepackage[nameinlink,noabbrev]{cleveref}

\theoremstyle{plain}
\newtheorem{theorem}{Theorem}[section]
\newtheorem{proposition}[theorem]{Proposition}

\theoremstyle{definition}

\theoremstyle{remark}

\newcommand{\F}{\mathbb F}
\newcommand{\cQ}{\mathcal Q}

\newcommand{\rad}{\operatorname{rad}}

\iclrfinalcopy

\title{Change the Product, Keep the Parameters: Associative Algebra Layers for Transformers}

\author{
Ilya Koziev\thanks{DAIMLD, Moscow, Russia. \texttt{inkoziev@gmail.com}}
\And
Ivan Oseledets\thanks{Artificial Intelligence Research Institute, Moscow, Russia}
}

\begin{document}

\maketitle
\fancyhead[L]{Paper under review}

\begin{abstract}
Fast matrix multiplication algorithms keep the product fixed and search for a cheaper way to evaluate it. We instead ask whether a Transformer's learned projections can use a different, cheaper product altogether. Building on an associative-algebra construction that replaces ordinary matrix multiplication with a sparser interaction table over the same weight blocks, we construct a family with quadratic arithmetic in the matrix dimension when the physical block size remains fixed, and derive finite-shape constraints for GPU execution. The construction is provably optimal for its bilinear rank by the Alder--Strassen bound and can be realized as row-typed rectangular projections compatible with causal masking and KV-cached decoding. We provide an empirical test of this approach by training two approximately 110M-parameter decoder-only Transformer LMs from the same recipe and 12.3B-token budget, differing only in their feed-forward layer: one uses ordinary dense matrix multiplication and the other uses the associative-algebra product. Across four prompt domains, the algebraic model achieves a 6.2--7.8\% increase in end-to-end generation throughput, while obtaining lower scores on all three reported downstream metrics. We treat these results as a feasibility and trainability check for the proposed approach at small scale, leaving further investigation to future work.
\end{abstract}

\section{Introduction}

A Transformer repeatedly applies learned linear maps of the form
\(Y=XW\).  After the feature dimensions are divided into groups, this
operation becomes a sum of ordinary rectangular block GEMMs.  The row--column
rule decides which activation block is multiplied by which weight block; the
learned entries of \(W\) decide the values of those interactions.  Most work on
fast matrix multiplication treats the row--column rule as fixed and searches
for a cheaper evaluation algorithm.

Strassen first separated the product from its algorithm: the exact product of
two \(2\times2\) block matrices can be evaluated with seven block
multiplications rather than the conventional eight~\citep{strassen1969}.
Algebraic complexity theory and recent automated searches pursue the same
direction~\citep{burgisser1997,fawzi2022,dupont2026omega}.  The target is still
ordinary matrix multiplication; the algorithm is allowed to change.

We ask the reverse question.  A hidden representation and its surrounding
layers are learned jointly, so must their intermediate product be ordinary
matrix multiplication?  What is a natural product that is weaker---and
cheaper---but still structured enough to use throughout a network?  Our
proposal keeps the learned weight blocks and changes the rule by which
activation and weight blocks interact.

Consider the smallest case, \(q=2\).  We first use square block arrays to
display the multiplication rule; later, each row of the same rule becomes a
rectangular Transformer projection.  Write
\begin{equation}
X=\begin{bmatrix}A&U\\V&D\end{bmatrix},\qquad
W=\begin{bmatrix}W_a&W_u\\W_v&W_d\end{bmatrix}.
\label{eq:two-block-inputs}
\end{equation}
All symbols in this display denote compatible matrix blocks, and every
juxtaposition below is one ordinary block GEMM.  We call the four positions in
each array \emph{block slots}.  Ordinary block multiplication gives
\begin{equation}
XW=\begin{bmatrix}
AW_a+UW_v & AW_u+UW_d\\
VW_a+DW_v & VW_u+DW_d
\end{bmatrix},
\label{eq:dense-two}
\end{equation}
and therefore uses eight block GEMMs.  One idea is simply to remove the two
terms that close the off-diagonal cycle:
\begin{equation}
X\star_2 W=\begin{bmatrix}
AW_a & AW_u+UW_d\\
VW_a+DW_v & DW_d
\end{bmatrix}.
\label{eq:six-product}
\end{equation}
Here ``remove \(UW_v\)'' has a literal implementation-level meaning: that GEMM
is absent from the fixed rule for every input.  The blocks \(U\) and \(W_v\)
remain; in particular, \(W_v\) is still learned and appears in \(DW_v\).
Likewise, \(W_u\) remains in \(AW_u\).  Thus all four weight blocks in
\eqref{eq:two-block-inputs} are independent learned parameters, while the new
rule uses six block GEMMs instead of eight.  The square display lists two row
maps at once.  In the Transformer realization, positions are assigned one of
two row types: a type-1 token supplies the feature groups \((A,U)\) and
evaluates the first output row, while a type-2 token supplies \((V,D)\) and
evaluates the second.  Each position therefore uses one row of the table.

Why remove exactly these two terms? Four matching-slot products would be
cheaper but provide only group-wise scaling. We instead seek a rule that
retains cross-group interactions and composes across layers.
For a fixed weight $W$, let $\Phi_W(X):=X\star_2 W$.
The rule in~\eqref{eq:six-product} satisfies
\begin{equation}
\Phi_{W_2}\!\left(\Phi_{W_1}(X)\right)
  =\Phi_{W_1\star_2 W_2}(X).
\label{eq:association-intro}
\end{equation}
\Needspace{8\baselineskip}
Thus two compatible linear maps defined by the rule compose into a map of the
same form.  The block identity
\(\smash{E=\left[\begin{smallmatrix}I&0\\0&I\end{smallmatrix}\right]}\)
satisfies \(E\star_2 W=W\) and \(X\star_2 E=X\).  This is a structural
statement about the parameterization: every learned block can change an output
for a suitable algebraic input.  The Transformer realization exposes the full
bank jointly across its row types.

A bilinear multiplication with these two properties is called an
\emph{associative unital algebra}.  Here the algebra is the fixed table used by
the layer: activations supply its first argument, learned weights its second,
and \(Y=\mu_{\mathcal A}(X,W)\).  We write \(\cQ_2\) for the four-slot law in
\eqref{eq:six-product}.

The pattern generalizes to a directed graph: diagonal slots are vertices,
retained off-diagonal slots are edges, and edge--edge products vanish.
\Cref{sec:algebras} gives the construction and proves associativity.

Can the retained rule be evaluated with fewer than six products?  Its
\emph{bilinear rank} is the smallest number of scalar multiplications in an
exact bilinear algorithm.  With compatible matrix blocks, the linear forms in
each rank-one term become block-linear combinations and each multiplication
becomes one block GEMM; actual time also depends on the rectangular shapes.
The Alder--Strassen theorem lower-bounds this rank for finite-dimensional
associative unital algebras~\citep{alder1981,blaser2000}.  For a graph table with
\(m\) slots and \(v\) vertices, its specialized bound is \(R\ge2m-v\), as
derived in \cref{sec:algebras}.  The present law has \(m=4\) and \(v=2\),
giving \(R(\cQ_2)\ge2\cdot4-2=6\).  The displayed evaluation uses six products
and is therefore optimal for this multiplication table.

The graph construction is a scalable example of the broader algebraic
design. Its size can grow with the layer width. With $q$ feature groups,
the arithmetic fraction is $(2q-1)/q^2$; holding the physical GEMM block
size fixed while increasing $q$ gives quadratic arithmetic for square
operands. The practical question is which group sizes preserve efficient
GPU execution. \Cref{sec:scaling} derives this scaling and its finite-size
constraints, and \cref{sec:kernel-benchmarks} measures representative
Transformer projection shapes.

We lift the slots to rectangular activation and weight blocks and use the
resulting row actions as position-wise Transformer projections. Training
evaluates all row types in parallel; autoregressive decoding evaluates only
the new position's row while retaining the standard KV cache.
\Cref{sec:layers,sec:autoregression} give the construction and costs.

The construction raises two empirical questions: can the restricted layers
be trained from scratch, and do their lower arithmetic counts improve
execution time? We test the recursive $\cQ_2^{\otimes2}$ law in the
feed-forward blocks of an approximately 110M-parameter language model,
against a parameter-matched dense baseline, after 12.3B training tokens each.
Generation throughput improves by 6.2--7.8\% across four prompt domains.
Scores on GSM8K, IFEval, and MBPP are lower for the algebraic model.
These results connect the algebraic savings to an observed
quality--throughput trade-off; the quality of larger algebraic constructions remains to be measured.

\section{Products as Architectural Choices}
\label{sec:background}

\subsection{A Fixed Table, Learned Weights}

Let $\mathcal A$ be a vector space with $m$ named slots. At the scalar level,
a bilinear multiplication table takes the form
\begin{equation}
[\mu_{\mathcal A}(x,w)]_k=\sum_{i,j}c_{ij}^{k}x_iw_j.
\label{eq:algebra-table}
\end{equation}
The coefficients $c_{ij}^{k}$ specify which activation slot interacts with
which weight slot and where the result is accumulated. They are fixed by
the architecture; $w$ is learned. Setting a coefficient to zero removes a
product from the rule. With compatible matrix blocks, that product is a GEMM.
Ordinary $q\times q$ matrix multiplication is one such table with $m=q^2$.

For $\Phi_w(x)=\mu_{\mathcal A}(x,w)$, associativity gives
\begin{equation}
\Phi_{w_2}\circ\Phi_{w_1}
=\Phi_{\mu_{\mathcal A}(w_1,w_2)}.
\label{eq:general-associativity}
\end{equation}
Thus compatible linear maps compose within the same family. A two-sided
identity $1_{\mathcal A}$ gives $\Phi_{1_{\mathcal A}}(x)=x$ and
$\Phi_w(1_{\mathcal A})=w$: every weight slot can affect an output.
For Transformer projections, a token supplies one row of the table and the
complete weight bank is exposed jointly across row types
(\cref{sec:layers}). Autoregressive causality additionally requires each
projection to act independently on token rows (\cref{sec:autoregression}).

\subsection{How Many Products Are Necessary?}

A rank-$r$ bilinear algorithm writes
\begin{equation}
\mu(x,w)=\sum_{\ell=1}^{r}
\gamma_\ell\alpha_\ell(x)\beta_\ell(w),
\label{eq:bilinear-decomp}
\end{equation}
where $\alpha_\ell,\beta_\ell$ are linear forms and $\gamma_\ell$ are output
vectors. The least $r$ is the bilinear rank $R(\mu)$; additions and fixed
linear combinations are counted separately~\citep{burgisser1997}.
For compatible block shapes, each scalar multiplication lifts to one GEMM.
The graph construction below uses individual blocks, so it also supports
unequal rectangular shapes without adding incompatible blocks.

For a finite-dimensional associative unital algebra, Alder--Strassen gives
\begin{equation}
R(\mathcal A)\ge 2\dim\mathcal A-t(\mathcal A),
\label{eq:alder-strassen}
\end{equation}
where $t(\mathcal A)$ counts its maximal two-sided
ideals~\citep{alder1981,blaser2000}. A two-sided ideal is a subspace stable
under multiplication from either side; maximal means maximal among proper
such subspaces. For our graph products, $t$ will equal the number of vertices.

The bound applies to a specified multiplication law. Coordinatewise
multiplication $[\mu(x,w)]_i=x_iw_i$ has rank $m$ and retains only within-slot
interactions. We seek a low-rank law that also couples feature groups.
Likewise, the conventional matrix-product counts $8$ and $64$ refer to its
block formula; they are not exact ranks, since already $R(M_2)=7$.

\section{Graph Products with Exact Cost}
\label{sec:algebras}

\subsection{The Graph Specifies the Product}

Let $G=(V,E)$ be a finite directed graph without self-loops. Give each
vertex a basis element $e_i$ and each directed edge a basis element $a_{ij}$.
Vertices represent diagonal feature slots; edges represent off-diagonal
interaction slots. Define $\cQ_G$ by the nonzero basis products
\begin{equation}
e_i e_i=e_i,\qquad e_i a_{ij}=a_{ij},\qquad a_{ij}e_j=a_{ij}.
\label{eq:graph-basis-products}
\end{equation}
All other basis products vanish. In coordinates, the complete rule is
\begin{align}
[\mu_{\cQ_G}(x,w)]_i&=x_iw_i,\label{eq:graph-product-vertex}\\
[\mu_{\cQ_G}(x,w)]_{ij}&=x_iw_{ij}+x_{ij}w_j.
\label{eq:graph-product-edge}
\end{align}
Each vertex contributes one product and each edge contributes two. The
two-vertex complete graph recovers~\eqref{eq:six-product}: its missing terms
multiply two edge slots.

\begin{proposition}[Structure]
The algebra $\cQ_G$ is associative and unital, with dimension $|V|+|E|$
and identity $\sum_i e_i$. The span $J$ of the edge elements is a two-sided
ideal satisfying $J^2=0$ and $\cQ_G/J\cong\F^{|V|}$.
\label{prop:structure}
\end{proposition}

The proof is given in \cref{app:graph-proofs}.

\begin{theorem}[Exact Bilinear Rank]
For every field $\F$,
\begin{equation}
R(\cQ_G)=|V|+2|E|.
\label{eq:graph-rank}
\end{equation}
\label{thm:graph-rank}
\end{theorem}

The proof is given in \cref{app:graph-proofs}.

\subsection{A Family with Adjustable Size}

For the complete directed graph on $q$ vertices, write $\cQ_q$ and
$A\star_q B:=\mu_{\cQ_q}(A,B)$. Arrange vertex slots on the diagonal
and edge slots off the diagonal. The rule becomes
\begin{align}
(A\star_q B)_{ii}&=A_{ii}B_{ii},\label{eq:q-product-diag}\\
(A\star_q B)_{ij}&=A_{ii}B_{ij}+A_{ij}B_{jj},\quad i\ne j.
\label{eq:q-product-off}
\end{align}
All $q^2$ slots are retained, and
\begin{equation}
\dim\cQ_q=q^2,\qquad R(\cQ_q)=2q^2-q.
\label{eq:q-rank}
\end{equation}
The two-group example is the smallest instance. Larger $q$ gives a
systematic arithmetic--interaction trade-off, studied in \cref{sec:scaling}.

Tensor products give another way to build larger algebras. The pretraining
pilot uses the tensor square of the two-group law, with 36 explicit products
at dimension sixteen. Its exact rank lies in $[28,36]$; the direct
four-group law attains 28. These are distinct interaction rules.
\Cref{app:tensor-square} gives the comparison and recursive row counts.

\section{Transformer Projections}
\label{sec:layers}

Let $X\in\F^{M\times K}$ and $W\in\F^{K\times N}$. Partition the
input and output features into $q$ groups, with widths $s_j$ and $r_j$,
so $K=\sum_j s_j$, $N=\sum_j r_j$, and
$W_{jk}\in\F^{s_j\times r_k}$. The block index is $q\times q$;
the physical blocks may be rectangular and unequal.

A token $x=(x_1,\ldots,x_q)$ at position $p$ selects a row
$i=\tau(p)$ by a fixed deterministic schedule. For example,
$\tau(p)=1+((p-1)\bmod q)$ cycles through the row types.
Its projection is
\begin{align}
y_i&=x_iW_{ii},\label{eq:typed-diag}\\
y_j&=x_iW_{ij}+x_jW_{jj},\quad j\ne i.
\label{eq:typed-off}
\end{align}
Compatible projections compose using the same block product; a diagonal
bank of identity maps gives the identity projection.

\begin{proposition}[Parameters and Rectangular Cost]
The bank stores $KN$ independent scalars. With
$D_{\rm diag}=\sum_j s_jr_j$, a type-$i$ row costs
\begin{equation}
C_i=D_{\rm diag}+s_i(N-r_i)
\label{eq:unequal-row-cost}
\end{equation}
MACs. If $M_i$ rows have type $i$, the batch and training costs are
\begin{align}
C_{\rm proj}^{\rm fwd}
 &=MD_{\rm diag}+\sum_i M_i s_i(N-r_i),
 \label{eq:rectangular-batch-cost}\\
C_{\rm proj}^{\rm train}&=3C_{\rm proj}^{\rm fwd}.
\label{eq:three-pass-projection-cost}
\end{align}
For equal feature groups, every row has dense-relative cost
\begin{equation}
\rho_q=\frac{2q-1}{q^2}.
\label{eq:rho}
\end{equation}
Across all types, the weight bank acts injectively.
\label{prop:parameters}
\end{proposition}

The gradient formulas and the three-pass count are given in \cref{app:vjp}; the parameter-coverage argument is immediate from the block support.

A token's local weight gradient reaches all diagonal blocks and its source
row of off-diagonal blocks. Other token types update the remaining rows.
These are interactions between feature groups within a token; attention
couples token positions. More groups reduce active computation while
restricting each position's linear map. The resulting map can still have
full rank when its diagonal blocks are invertible.

\paragraph{Projection Shapes.}
A gated FFN uses two $D\times F$ weights and one $F\times D$ weight.
For attention, let $H_q=n_qd_h$ and $H_{kv}=n_{kv}d_h$ denote aggregate
query and key/value widths. The Q/K/V/O weight shapes are
$D\times H_q$, $D\times H_{kv}$, $D\times H_{kv}$, and $H_q\times D$.
The same row law applies to each shape. Gates, pointwise activations,
normalization, residual connections, and biases retain their usual
definitions. The pretraining pilot replaces only the FFN projections.

\section{Scaling and Practical Implementation}
\label{sec:scaling}

The algebra size determines how much arithmetic is removed; the physical
block sizes determine how efficiently the remaining products run.
These are separate choices. A fixed small algebra gives a constant-factor
reduction. Allowing its size to grow changes the asymptotic cost.

\subsection{Growing the Algebra}

Consider square operands of side $n=qb$, using the complete graph algebra
on $q$ vertices and ordinary $b\times b$ products inside each slot.
The block-valued algebra is
$\mathcal A_{q,b}=\cQ_q\otimes M_b(\F)$, of dimension $n^2$.

\begin{proposition}[Quadratic Arithmetic at Fixed Block Size]
The direct block algorithm uses
\begin{equation}
C(n,q)=(2q^2-q)(n/q)^3
=\left(\frac2q-\frac1{q^2}\right)n^3
\label{eq:scaling-square}
\end{equation}
MACs. For any fixed positive block size $b$, choosing $q=n/b$ gives
\begin{equation}
C(n,n/b)=2bn^2-b^2n=\Theta(n^2).
\label{eq:scaling-quadratic}
\end{equation}
Moreover, $R(\mathcal A_{q,b})\ge2n^2-q$.
\label{prop:quadratic-scaling}
\end{proposition}

The rank bound follows from Alder--Strassen; see \cref{app:scaling-proof}.

The quadratic order therefore requires no scalar-size leaves: $b$ can remain
large enough for a conventional GEMM. The multiplication law changes with
$q$. At fixed $q$, \eqref{eq:scaling-square} remains cubic; for
$q=\Theta(n^\alpha)$, it is $\Theta(n^{3-\alpha})$.
Increasing $q$ also restricts each row's feature interactions and changes
how often off-diagonal parameters are used. Computational scaling alone
does not establish a quality-preserving scaling law.

\subsection{Finite Rectangular Shapes}

For a projection $X\in\F^{M\times K}$ with $K\times N$ weights and equal
groups, the exact active MAC count is
\begin{equation}
C_{\rm proj}(M,K,N;q)=\frac{2q-1}{q^2}MKN.
\label{eq:scaling-rectangular}
\end{equation}
The type-specific terms have dimensions approximately
$(M/q)\times(K/q)\times(N/q)$. If efficient leaf GEMMs require
at least $m_0,k_0,n_0$ along these axes, a useful screening condition is
\begin{equation}
q\ \le\ \min\{M/m_0,\ K/k_0,\ N/n_0\}.
\label{eq:hardware-q}
\end{equation}
The admissible $q$ must also divide the feature widths for equal groups.
The criterion provides an initial size filter. For example,
requiring all three axes to be at least 128 permits $q\le16$ at
$M=K=N=2048$, and $q\le32$ at side 4096.
Decode requires a separate matrix--vector implementation because its row
dimension is small.
At non-aligned sizes, the executed tiles include masked rows and columns;
their arithmetic can substantially exceed the active count in
\eqref{eq:scaling-rectangular}.

Let $P_d,P_a(q)$ be the effective dense and algebraic compute rates,
$\eta_q=P_a(q)/P_d$, and let $\delta_q$ be additional overhead divided by
dense time. In a compute-dominated model, speedup requires
\begin{equation}
\rho_q/\eta_q+\delta_q<1.
\label{eq:crossover}
\end{equation}
Smaller blocks can reduce $\eta_q$ enough to offset the arithmetic saving.
Ignoring overhead, $q>2/\eta_q$ is sufficient; along the fixed-$b$
family this reads $n>2b/\eta_q$. The effective rate must be measured
at the resulting leaf shapes.
Memory traffic remains a constraint because a batch exposing all row types
uses the full weight bank.

\subsection{Kernels}

We extend the endpoint-product schedules of our existing block-algebra
kernels to arbitrary $q$ and rectangular feature widths. One forward kernel
accumulates both endpoint products before writing the output. A second
schedule computes the diagonal maps over all rows, then the type-dependent
maps; this improves reuse when each type has few rows. A dedicated decode
kernel performs vector reductions. The activation and weight gradients have
separate kernels with one owner per output element.

These kernels consume the canonical weight bank. Tile selection is calibrated
on each shape and evaluated with separate timing trials. The next section
reports where the arithmetic reduction produces a measured benefit.
Each $q$ defines a different architecture. A deployed model keeps $q$ fixed
across training, prefill, and decoding; kernel tiles and schedules can adapt
to the workload.

\section{Training and Inference}
\label{sec:autoregression}

Training evaluates all token rows in parallel; autoregressive decoding
evaluates one new row per sequence. Both regimes use the same weights and
the same position-type schedule.

\subsection{One Row Action in Both Regimes}

For a support set $\mathcal B(i,j)$ determined by the algebra, a row of
type $i$ computes
\begin{equation}
y_j=\sum_{k\in\mathcal B(i,j)}x_kW_{kj}.
\label{eq:recursive-row-action}
\end{equation}
For the complete graph construction, $\mathcal B(i,j)=\{i,j\}$ as a set.
The recursive pilot uses the endpoint sets in \cref{app:tensor-square}.
Algorithm~\ref{alg:projection} applies to both.

\begin{algorithm}[t]
\caption{Batched or one-token algebraic projection.}
\label{alg:projection}
\begin{algorithmic}[1]
\REQUIRE Rows $X$, positions $p_b$, weights $W$, support $\mathcal B$
\STATE $Y\gets0$
\FOR{each row type $i$}
  \STATE $I_i\gets\{b:\tau(p_b)=i\}$
  \FOR{each output group $j$ and $k\in\mathcal B(i,j)$}
    \STATE $Y_j[I_i]\gets Y_j[I_i]+X_k[I_i]W_{kj}$
  \ENDFOR
\ENDFOR
\STATE \textbf{return} $Y$
\end{algorithmic}
\end{algorithm}

\subsection{Training}

Each term in Algorithm~\ref{alg:projection} is a rectangular GEMM.
The two backward maps use the same support, so matrix-product MACs for
forward plus backward equal three forward evaluations
(\cref{prop:parameters}; \cref{app:vjp} gives the direct-law gradients).
For a gated MLP with widths $D,F$, the three projections have shapes
$D\times F,D\times F,F\times D$. On $B$ token rows, their training
cost falls from $9BDF$ to $9\rho BDF$. Activations, gates, normalization,
and residual operations are additional.

\begin{proposition}[Causality]
A fixed position-type schedule and row-local algebraic projections,
combined with the usual triangular attention mask, preserve
autoregressive causality.
\label{prop:causality}
\end{proposition}
The proof is the standard induction over layers; see \cref{app:causality-proof}.

\subsection{One-Token Inference}

At position $p$, the decoder appends its key/value to the cache, attends to
allowed positions, and applies the output and FFN projections. The FFN uses
Algorithm~\ref{alg:projection} with row type $\tau(p)$; an absolute-position
schedule keeps this type consistent across training, prefill, and decoding.

For $N$ cached positions, per-head width $d_h$, and $n_{kv}$ key/value heads,
\begin{equation}
K^{\rm cache},V^{\rm cache}\in\R^{N\times n_{kv}\times d_h}.
\label{eq:icml-cache}
\end{equation}
Storage is $2NH_{kv}$ scalars, with $H_{kv}=n_{kv}d_h$.
Let $H_q=n_qd_h$ and let $g=2$ for a gated MLP ($g=1$ for a two-map MLP).
With algebraic FFN projections and dense attention, one decoder block costs
\begin{equation}
C_{\rm dec}=2DH_q+2DH_{kv}+\rho(g+1)DF+2NH_q
\label{eq:icml-decode-total}
\end{equation}
MACs in projections and attention contractions. The dense baseline sets
$\rho=1$. In the tested configuration, $H_q=H_{kv}=D$, $F=2D$, and
$\rho=9/16$. Its decoder-block projection fraction is
\[
\frac{4D^2+(9/16)\,6D^2}{10D^2}=\frac{59}{80}.
\]
Attention contractions and the language-model head further dilute the FFN
saving. These counts exclude memory traffic and elementwise operations.

\paragraph{Implementation and Extensions.}
The smaller products can use off-the-shelf GEMMs; specialized kernels can
fuse their products and accumulations. The pilot uses CUDA-graph decode
paths for both models. Graph capture reduces launch overhead, while
single-token computation can still be limited by memory bandwidth.
\Cref{sec:attention-details} develops the separate extension to algebraic
attention scores, including its cache and normalization formulas; the
experiment uses dense attention.

\section{Experiments}
\label{sec:experiments}
\subsection{Kernel Scaling on Transformer Shapes}
\label{sec:kernel-benchmarks}

We measure both FFN projection directions at the public dimensions of
Qwen3-1.7B, 4B, 8B, a Qwen3-30B-A3B expert, and a DeepSeek-V3
expert~\citep{qwen3report,deepseekv3}. The corresponding $(D,F)$ pairs are
$(2048,6144)$, $(2560,9728)$, $(4096,12288)$, $(2048,768)$, and
$(7168,2048)$. An expert's $M$ counts tokens already routed to that expert.
We vary $M\in\{1,32,512,2048\}$ and $q\in\{4,8,16,32,64\}$.

The study uses synthetic BF16 tensors on one H200 NVL, FP32 accumulation,
CUDA Graph replay for both implementations, and preallocated outputs.
The primary comparison is ordinary $XW$ against the algebraic projection,
without activation or FFN fusion. Calibration selects the faster available
cuBLAS/cuBLASLt path for the dense shape and the schedule and tile for the
algebraic kernel; separate alternating trials provide the reported medians.
Forward and both gradient maps are checked against independent formulas.

Warm replay and a separate cache-flushed measurement characterize different
reuse regimes. The reported speedups are kernel-level projection/FFN speedups,
not end-to-end Transformer speedups; whole-model throughput, MoE dispatch, and
inter-device communication require separate measurements. Fully realizing the
acceleration potential will also require fused kernels; we leave such kernel
fusion to future work.
The additional gated-block comparison merges gate/up on both sides and
uses the same activation kernel; its results are reported in the appendix.
\begin{table*}[t]
\centering
\caption{Pure expansion projections, BF16 on H200 NVL. The selected $q_\ast$ has the lowest warm Q-kernel median in the tested grid: $\{4,8,16,32,64\}$ for the first five shapes and $\{32,64,128,256\}$ for Qwen3-32B. Block dimensions $(m,k,n)=(M/q_\ast,K/q_\ast,N/q_\ast)$ show where subdivision stops, before GPU-tile padding. Times are microseconds; speedups are kernel speedups (dense kernel / algebraic kernel), reported for warm / cache-flushed runs. $\dagger$ marks the largest tested $q$.}
\label{tab:kernel-products}
\small
\begin{tabular*}{\textwidth}{@{\extracolsep{\fill}}lrrcrrr@{}}
\toprule
Model shape & $M$ & $q_\ast$ & Block $(m,k,n)$ & Dense $\mu$s & Q $\mu$s & Speedup\\
\midrule
Qwen3-1.7B & 2048 & $32$ & $(64,64,192)$ & 75.7 & 24.9 & $3.04\times$ / $2.84\times$\\
Qwen3-4B & 2048 & $32$ & $(64,80,304)$ & 160.6 & 51.1 & $3.14\times$ / $2.94\times$\\
Qwen3-8B & 2048 & $32$ & $(64,128,384)$ & 320.7 & 63.6 & $5.04\times$ / $4.52\times$\\
Qwen3-30B-A3B expert & 2048 & $16$ & $(128,128,48)$ & 13.2 & 8.4 & $1.57\times$ / $1.30\times$\\
DeepSeek-V3 expert & 2048 & $16$ & $(128,448,128)$ & 82.3 & 28.3 & $2.91\times$ / $2.47\times$\\
Qwen3-32B & 2048 & $32$ & $(64,160,800)$ & 849.2 & 167.8 & $5.06\times$ / $4.86\times$\\
Qwen3-32B & 8192 & $128$ & $(64,40,200)$ & 3337.4 & 400.7 & $8.33\times$ / $8.60\times$\\
Qwen3-32B & 32768 & $256^\dagger$ & $(128,20,100)$ & 13458.6 & 1169.4 & $11.51\times$ / $11.70\times$\\
\bottomrule
\end{tabular*}
\end{table*}
The extended Qwen3-32B projection sweep and fixed-block square-product
measurements are reported in \cref{app:large-products}. The main-table
results already show the dependence of useful subdivision on workload size
and physical block shape.

\subsection{Small Pretraining Study}

We train two approximately 110M-parameter decoder-only language models from
scratch on 12.3B tokens each, with one run per model. Both use 12 pre-norm
blocks, width $D=768$, 32 attention heads, learned positional embeddings,
context length 1536, and the GPT-2 tokenizer with vocabulary size 50,257.
The bias-free SwiGLU FFN has width $F=1536$ and three weight matrices
of shapes $D\times F,D\times F,F\times D$. The dense model uses ordinary
projections; the algebraic model uses the recursive
$\cQ_2^{\otimes2}$ action~\eqref{eq:recursive-row-action}.
Attention, embeddings, and the language-model head are dense in both.

Both runs use two NVIDIA H100 GPUs with a per-GPU batch size of 40, AdamW with peak learning rate $3\times10^{-4}$, 500 warmup steps, and cosine decay to 10\% of the peak.
Prompt and response fields from the same public instruction-data mixture are concatenated for training.
Full hyperparameters and the reported corpus manifest are in \cref{app:empirical-details,app:training_dataset}.
The algebraic FFN's projection MAC fraction is $9/16$.

\paragraph{Generation Throughput.}

Both models use CUDA-graph single-token decode paths. We measure end-to-end
generation throughput including tokenization and sampling across four prompt
domains (\cref{tab:throughput}). The algebraic model has higher throughput
in every domain, by 6.2--7.8\%. The unweighted mean of the four throughput
ratios is $1.070$.

\begin{table}[t]
\centering
\caption{Measured generation throughput. Gain is the relative increase
over the dense model within each domain.}
\label{tab:throughput}
\small
\begin{tabular}{@{}lrrr@{}}
\toprule
Domain & Dense tok/s & Algebraic tok/s & Gain\\
\midrule
Code & 987.8 & 1054.5 & 6.8\%\\
Math & 991.4 & 1060.6 & 7.0\%\\
QA & 958.6 & 1033.1 & 7.8\%\\
WMT ru--en & 985.9 & 1047.1 & 6.2\%\\
\bottomrule
\end{tabular}
\end{table}

Generated lengths differ between models: for example, the mean code response is 195.5 tokens for dense and 182.4 for algebraic. The measurements therefore describe the realized generation workloads. Prompt and response lengths are reported in
\cref{tab:generation-speed-baseline,tab:generation-speed-algebraic}.

\paragraph{Downstream Quality.}

We use the evaluation harness~\citep{lmevalharness} to evaluate
GSM8K~\citep{cobbe2021gsm8k}, IFEval~\citep{zhou2023ifeval}, and
MBPP~\citep{austin2021mbpp} on the same accelerated inference paths.
GSM8K uses five-shot prompting; IFEval and MBPP use zero-shot evaluation.
\Cref{tab:downstream} reports a concise comparison; all six metric values
and available standard errors are in \cref{tab:full-downstream}.

\begin{table}[t]
\centering
\caption{Downstream scores in percent. Higher is better. The appendix
reports the complete metric set and evaluation standard errors.}
\label{tab:downstream}
\small
\begin{tabular}{@{}llrr@{}}
\toprule
Task & Metric & Dense & Algebraic\\
\midrule
GSM8K & EM, flexible & 3.18 & 2.20\\
IFEval & Instance, loose & 34.5 & 31.1\\
MBPP & pass@1 & 6.60 & 3.80\\
\bottomrule
\end{tabular}
\end{table}

The algebraic model scores lower on all three reported metrics.
The GSM8K gap is 0.98 points, loose instance-level IFEval gap is 3.4 points; the MBPP pass@1 gap is 2.8 points.
The pilot establishes that the recursive layer can be trained and used
for faster generation under this recipe, with a measurable quality cost.
\section{Related Work}
\label{sec:related}

\paragraph{Fast matrix multiplication.}
Strassen showed that classical matrix multiplication is arithmetically
suboptimal~\citep{strassen1969}; tensor rank and algebraic complexity provide
the framework for subsequent bounds~\citep{alder1981,burgisser1997,blaser2000}.
Recent work improves the matrix-multiplication exponent while retaining the
same target product~\citep{dupont2026omega}. We instead change the product.

\paragraph{Alternative algebras and structured layers.}
AlgebraNets and algebraic neural networks replace or generalize ordinary
multiplication through alternative or finite-dimensional algebra laws
~\citep{hoffmann2020algebranets,parada2021algebraic,niemczynowicz2026tensorial}.
Hypercomplex multiplication, Tensor Train, Monarch, and Block Tensor-Train
methods impose structure primarily through parameter factorization or
reduction~\citep{zhang2021phm,oseledets2011,novikov2015tensorizing,dao2022monarch,qiu2024btt}.
Our setting keeps exactly $d_{\rm in}d_{\rm out}$ learned scalars and changes
only the multiplication table, selected by certified scalar bilinear rank.

\paragraph{Other efficient Transformers.}
Matmul-free models remove dense matrix multiplications through alternative representations~\citep{zhu2024matmulfree}; MoE models use content-dependent expert selection~\citep{fedus2022switch}. Our method uses fixed type-dependent sparsity with the full weight bank.
\citet{dao2022flashattention} show that memory traffic also affects wall-clock attention time; our runtime experiment keeps attention dense.

\section{Discussion and Limitations}
\label{sec:limitations}

The experiment tests one recursive FFN law, one model scale, and one training
run per model. Larger algebras, algebraic attention, larger models, and
repeated seeds remain untested. Projection benchmarks use synthetic tensors;
evaluation standard errors quantify task-example uncertainty, not
training-run variation.

The throughput gain is smaller than the FFN arithmetic reduction because
attention, the LM head, sampling, tokenization, and memory traffic also
contribute. Generation lengths differ, so fixed-length repeated full-model
timings are needed to isolate the layer effect.

Equal parameter storage permits different function classes, while
associativity guarantees composition rather than quality preservation.
Thus the results do not establish an advantage at equal quality or training
time.

\section{Conclusion}

Associative algebras make the multiplication law an architectural choice.
A growing graph family retains the dense-sized weight bank and reduces
square-product arithmetic to quadratic order at a fixed physical block size.

The kernel measurements identify this arithmetic--latency trade-off on
representative Transformer shapes. The pretraining pilot separately shows
a generation-throughput gain accompanied by lower downstream scores.

\subsection*{AI use statement}

We used generative AI tools for spellchecking and grammar checking, for
\LaTeX-related formatting, and for drafting parts of the paper. All
AI-assisted text and formatting were reviewed by the authors, who take full
responsibility for the final content, claims, and presentation of this work.

\label{sec:main-end}

\bibliographystyle{iclr2027_conference}
\bibliography{\paperroot refs}

\clearpage

\appendix
\section{Direct Associativity Check}

Write an element of $\cQ_q$ as $(d,n)$ with $d\in\F^q$ diagonal and $n$ off-diagonal.  The product is
\begin{equation}
(d,n)\star_q(d',n')=(dd',\,dn'+nd'),
\label{eq:semidirect-product}
\end{equation}
where left and right diagonal multiplication act on the source and target of each arrow; explicitly, $(dn')_{ij}=d_i n'_{ij}$ and $(nd')_{ij}=n_{ij}d'_j$.  Since the left and right actions commute, both parenthesizations have diagonal part $dd'd''$.  Let
\[
\Xi=dd'n''+dn'd''+nd'd''.
\]
Expanding either parenthesization gives
\begin{align*}
((d,n)\star_q(d',n'))\star_q(d'',n'')&=(dd'd'',\Xi),\\
(d,n)\star_q((d',n')\star_q(d'',n''))&=(dd'd'',\Xi),
\end{align*}
so the law is associative.  The element $(\bm 1,0)$ is the unit.  This calculation complements the basis-element argument in \cref{prop:structure}.

\section{Vector--Jacobian Products}
\label{app:vjp}

For one token of type $i$, let $\bar y_j$ denote the gradient of the loss with respect to output group $j$.  The typed projection~\eqref{eq:typed-diag}--\eqref{eq:typed-off} gives
\begin{align}
\bar x_i &\mathrel{+}= \bar y_iW_{ii}^{\top}
+\sum_{j\neq i}\bar y_jW_{ij}^{\top},\label{eq:vjp-xi}\\
\bar x_j &\mathrel{+}= \bar y_jW_{jj}^{\top},\qquad j\neq i,\label{eq:vjp-xj}\\
\bar W_{ii} &\mathrel{+}=x_i^{\top}\bar y_i,\label{eq:vjp-wii}\\
\bar W_{ij} &\mathrel{+}=x_i^{\top}\bar y_j,\qquad j\neq i,\label{eq:vjp-wij}\\
\bar W_{jj} &\mathrel{+}=x_j^{\top}\bar y_j,\qquad j\neq i.
\label{eq:vjp-wjj}
\end{align}
Summing these expressions over token types yields gradients for every block.  The displayed vector--Jacobian products preserve exactly the block support of the forward multiplication table.  Hence $dX$ and $dW$ each have the same matrix-product MAC count as the forward projection, which proves~\eqref{eq:three-pass-projection-cost}.

\section{The Tensor-Square Candidate}
\label{app:tensor-square}

Let $J=\rad(\cQ_2)$.  Although $\cQ_4$ and $\cQ_2^{\otimes2}$ are both sixteen-dimensional and have semisimple quotient $\F^4$, they are not isomorphic.  In $\cQ_4$, the radical satisfies $\rad(\cQ_4)^2=0$.  In the tensor square,
\[
\rad(\cQ_2^{\otimes2})
=J\otimes\cQ_2+\cQ_2\otimes J,
\]
since the right-hand side is nilpotent and the quotient is
$(\cQ_2/J)\otimes(\cQ_2/J)\cong\F^4$.  Its square contains $J\otimes J\neq0$ because $J\cQ_2=\cQ_2J=J$.  The tensor-square law therefore permits some two-step radical interactions that strict $\cQ_4$ removes.  The 36-versus-28 comparison concerns two non-isomorphic products on the same number of coordinates.

\subsection{Rectangular Counts for Recursive Products}

For $\mathcal A_L=\cQ_2^{\otimes L}$, types and feature-group indices are
bit strings in $\{0,1\}^L$. A token of type $\alpha$ uses intermediate
groups $\beta$ whose coordinates come from the corresponding endpoints
$\alpha,\delta$. With input widths $s_\beta$ and output widths $r_\delta$,
\begin{equation}
C_\alpha^{(L)}=
\sum_\delta
\sum_{\substack{\beta:\,\beta_k\in\{\alpha_k,\delta_k\}\\k=1,\ldots,L}}
s_\beta r_\delta.
\label{eq:recursive-unequal-cost}
\end{equation}
For uniform types, a block $(\beta,\delta)$ occurs with probability
$2^{-h(\beta,\delta)}$, where $h$ is the Hamming distance. Hence
\begin{equation}
\bar C^{(L)}
=\sum_{\beta,\delta}2^{-h(\beta,\delta)}s_\beta r_\delta.
\label{eq:recursive-uniform-cost}
\end{equation}
Equal group widths give ratio $(3/4)^L$ relative to a dense projection.
At $L=2$, the full product uses $4\cdot1+8\cdot2+4\cdot4=36$ block
products. More generally the recursive algorithm uses $6^L$ products,
establishing $R(\cQ_2^{\otimes L})\le6^L$.

\section{Experimental Details}
\label{app:empirical-details}

\subsection{Training Configuration}

Both models are trained from scratch using the same instruction-data
mixture described in \autoref{app:training_dataset}, tokenizer, initialization scheme, and optimization recipe.
Each model is trained for 100,000 optimizer steps on two NVIDIA H100 GPUs,
with a per-GPU batch size of 40 and sequence length 1536. This corresponds
to 12.288B token slots, or approximately 12.3B consumed training tokens.

\begin{table}[t]
\centering
\caption{Training configuration shared by both models.}
\label{tab:hyperparams}
\small
\begin{tabular}{@{}lr@{}}
\toprule
Quantity & Value\\
\midrule
Decoder blocks & 12\\
Model width & 768\\
FFN width & 1536\\
Attention heads & 32\\
Context length & 1536\\
Vocabulary size & 50,257\\
Parameters & approximately 110M\\
Training tokens & 12.3B\\
Hardware & 2 NVIDIA H100 GPUs\\
Microbatch size & 40 per GPU\\
Gradient accumulation & 1\\
Maximum optimizer steps & 100,000\\
Warmup steps & 500\\
Peak learning rate & $3\times10^{-4}$\\
Final learning-rate fraction & 0.1\\
Weight decay & 0.1\\
Adam $(\beta_1,\beta_2)$ & $(0.9,0.95)$\\
Adam $\epsilon$ & $10^{-8}$\\
Gradient clipping norm & 1.0\\
\bottomrule
\end{tabular}
\end{table}

Both models contain 110,631,168 trainable parameters.

The software versions are CUDA 13.0, PyTorch 2.13.0, and
\texttt{transformers} 4.54.1. The tokenizer is GPT-2 BPE; positional
embeddings are learned. FFN projections are bias-free.

Both models use BF16 autocast for forward and backward computation, while
trainable parameters are stored in FP32. Evaluation loss curves for both
models are shown in \autoref{fig:eval_loss}.

\begin{figure}
    \centering
    \includegraphics[width=0.5\linewidth]{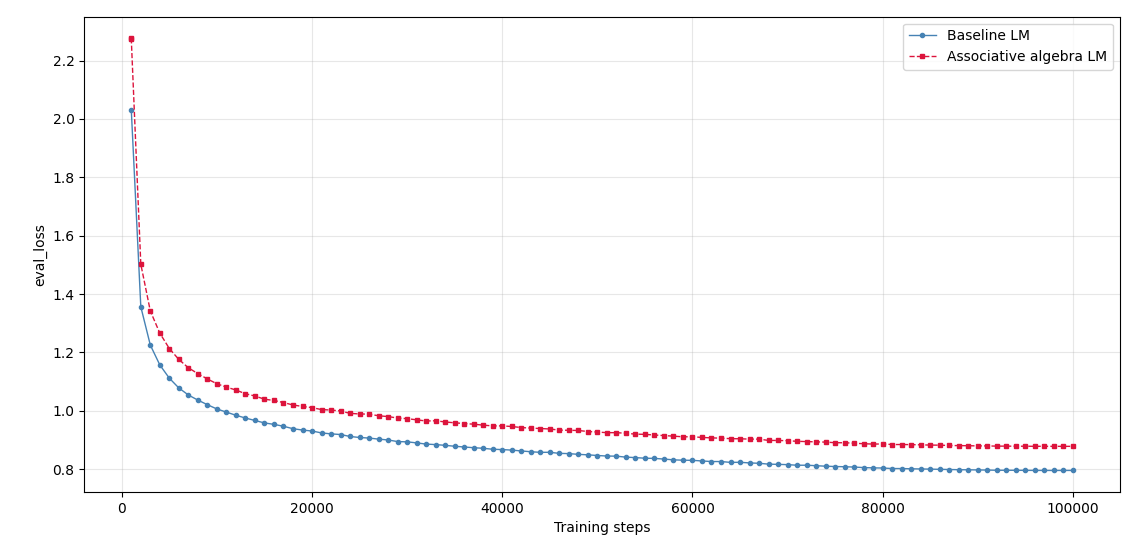}
    \caption{Evaluation loss curves for the baseline and associative algebra LMs.}
    \label{fig:eval_loss}
\end{figure}

\subsection{Generation Speed Measurement Protocol}

Generation speed was measured separately on four domains: \emph{code}
(code generation), \emph{math} (solving mathematical problems),
\emph{qa} (general questions), and \emph{wmt ru--en} (Russian-to-English
translation). Each domain contains 100 unique prompts.

Responses were generated with sampling using temperature $=0.8$,
top-$p=0.95$, and top-$k=50$, with a maximum of 256 generated tokens.
All generation was performed in BF16 precision. Measurements were
conducted four times for each model and domain. The first run was excluded
from the reported averages to reduce the effect of CUDA and GPU-cache
warm-up; the remaining three measurements were used for the reported means.

The margin of error for each mean was computed using the Student's
$t$-distribution with a 95\% confidence interval. The reported generation
speed includes the complete measured generation workload, including
tokenization and sampling.

\begin{table*}[t]
\centering
\caption{Generation speed measurements for the baseline model. Each domain
contains 100 unique prompts. The first of four measurement runs was
excluded as a warm-up; the three retained runs are shown in the
\emph{Measurements} column.}
\label{tab:generation-speed-baseline}
\small
\begin{tabular*}{\textwidth}{@{\extracolsep{\fill}}lrrrrr@{}}
\toprule
Domain &
Mean prompt &
Mean response &
Measurements &
Mean speed &
95\% CI margin\\
&
tokens &
tokens &
tok/sec &
tok/sec &
of error\\
\midrule
Code & 42.9 & 195.5 & 987.5, 988.2, 987.6 & 987.8 & 1.0\\
Math & 54.3 & 223.8 & 990.1, 993.7, 990.3 & 991.4 & 5.0\\
QA & 17.1 & 92.6 & 954.0, 962.8, 959.0 & 958.6 & 11.0\\
WMT ru--en & 263.8 & 190.0 & 988.0, 986.0, 983.7 & 985.9 & 5.2\\
\bottomrule
\end{tabular*}
\end{table*}

\begin{table*}[t]
\centering
\caption{Generation speed measurements for the associative algebra model.
Each domain contains 100 unique prompts. The first of four measurement
runs was excluded as a warm-up; the three retained runs are shown in the
\emph{Measurements} column.}
\label{tab:generation-speed-algebraic}
\small
\begin{tabular*}{\textwidth}{@{\extracolsep{\fill}}lrrrrr@{}}
\toprule
Domain &
Mean prompt &
Mean response &
Measurements &
Mean speed &
95\% CI margin\\
&
tokens &
tokens &
tok/sec &
tok/sec &
of error\\
\midrule
Code & 42.9 & 182.4 & 1055.5, 1053.7, 1054.2 & 1054.5 & 2.2\\
Math & 54.3 & 227.8 & 1058.9, 1061.9, 1060.9 & 1060.6 & 3.8\\
QA & 17.1 & 110.8 & 1031.9, 1034.6, 1032.6 & 1033.1 & 3.5\\
WMT ru--en & 263.8 & 204.9 & 1047.3, 1048.0, 1046.0 & 1047.1 & 2.6\\
\bottomrule
\end{tabular*}
\end{table*}

\begin{table}[t]
\centering
\caption{Associative algebra model generation speed relative to the
baseline model.}
\label{tab:generation-speedup}
\small
\begin{tabular}{@{}lr@{}}
\toprule
Domain & Relative speedup (\%)\\
\midrule
Code & 6.8\\
Math & 7.0\\
QA & 7.8\\
WMT ru--en & 6.2\\
\bottomrule
\end{tabular}
\end{table}

\subsection{Downstream Tasks Evaluation Protocol}

Downstream evaluation was performed using
\texttt{lm-eval-harness==0.4.9.1} on the generative tasks GSM8K, MBPP,
and IFEval. The maximum number of generated tokens was 512, the batch
size was 1, and generation was performed in BF16 precision. All other
generation parameters were left at their default values for the
\texttt{lm\_eval.simple\_evaluate()} call.

\subsection{Complete Downstream Results}

\Cref{tab:full-downstream} reports the complete set of supplied downstream
results on their original $[0,1]$ scale. The standard errors are those
reported by the evaluation harness; ``N/A'' indicates that no standard
error was supplied.

\begin{table}[t]
\centering
\caption{Complete downstream results on the original $[0,1]$ scale.
EM denotes exact match.}
\label{tab:full-downstream}
\small
\begin{tabular}{@{}llrr@{}}
\toprule
Task & Metric & Dense & Algebraic\\
\midrule
GSM8K &
EM, flexible-extract &
0.0318 $\pm$ 0.0048 &
0.0220 $\pm$ 0.0040\\
MBPP &
pass@1 &
0.0660 $\pm$ 0.0111 &
0.0380 $\pm$ 0.0086\\
IFEval &
Instruction-level loose accuracy &
0.3453 &
0.3106\\
\bottomrule
\end{tabular}
\end{table}

The underlying evaluation outputs were
0.0318423 and 0.0219864 for GSM8K,
0.066 and 0.038 for MBPP, and
0.345324 and 0.310552 for IFEval, respectively. The corresponding
reported standard errors were 0.004836 and 0.004039 for GSM8K and
0.011115 and 0.008559 for MBPP; no standard error was supplied for
IFEval.

\subsection{Reproducibility Status}

Both models were trained with the same model architecture, optimization
hyperparameters, data pipeline, and training budget, with the dense and
associative algebra models differing in the multiplication law used by
the FFN projections. The training run uses two NVIDIA H100 GPUs and
100,000 optimizer steps with a per-GPU batch size of 40, corresponding
to approximately 12.3B consumed token slots.

The available GPT-2 pipeline uses two recursive levels, with feature splits
$768\to384\to192$. For zero-based position $p$, its four-valued row type is
$\lfloor(p\bmod768)/192\rfloor$, corresponding to the two-bit index
$\alpha$. This schedule assigns consecutive intervals to each type.
The projection microbenchmarks use cyclic types to cover growing group
counts uniformly. The $9/16$ arithmetic fraction holds for the stated
tensor-square row action with equal groups, irrespective of the type
frequencies.

The generation measurements use four domains with 100 unique prompts
per domain, four measurement runs, and three retained runs after
warm-up exclusion. Margins of error are based on Student's $t$-distribution
at a 95\% confidence level. Downstream evaluation uses
\texttt{lm-eval-harness==0.4.9.1} with the settings specified above.

\section{Kernel Benchmark Details}
\label{app:kernel-benchmarks}
The primary benchmark isolates the product. The dense control selects the faster cuBLAS/cuBLASLt preference during calibration. The additional SwiGLU test applies identical gate/up merging and activation code to both methods. All tensors are synthetic BF16 with FP32 accumulation. Timings use H200 NVL, PyTorch 2.13.0+cu130, and Triton 3.7.1.
For each measured configuration, twelve held-out trials alternate the order of the two methods. Warm trials repeatedly replay the same graph. In the flushed protocol, a 512 MiB buffer is touched before each single replay, outside timing. The protocols also differ in sustained duty cycle; they characterize two reuse regimes. At $M=1$, warm replay repeats one row type. Full decoder timings require its actual sequence of types and cache accesses.
Small-shape verification passed 86 forward/backward checks, including non-power-of-two group counts, with maximum relative $L_2$ error 0.00208. Chosen forward tiles and sampled large-shape gradient entries are checked independently. Algebraic gradient kernels use a fixed $(32,64,32)$ tile with four warps; only forward schedules are calibrated. The complete FFN is compared against a PyTorch composition.
Concurrent compute processes are checked before and after each measurement. An interrupted preliminary screen is excluded. Performance counters were unavailable under the server permissions, so the report makes no counter-based bottleneck attribution. MoE measurements exclude dispatch and communication. DeepSeek dimensions are evaluated in BF16; its production FP8 implementation is outside this comparison.
\begin{table*}[t]
\centering
\caption{Complete SwiGLU forward at fixed $q=32$, BF16 on H200 NVL. Both implementations merge gate/up. Times are warm-replay medians in microseconds. The last column reports the separately measured cache-flushed speedup. These are isolated FFNs on public model shapes.}
\label{tab:kernel-ffn}
\small
\begin{tabular*}{\textwidth}{@{\extracolsep{\fill}}lrrrrrr@{}}
\toprule
Model shape & $D$ & $F$ & $M$ & Dense $\mu$s & Algebra $\mu$s & Speedup (warm / flushed)\\
\midrule
Qwen3-1.7B & 2048 & 6144 & 1 & 30.2 & 8.7 & $3.45\times$ / $2.46\times$\\
Qwen3-1.7B & 2048 & 6144 & 2048 & 273.7 & 107.4 & $2.55\times$ / $2.45\times$\\
Qwen3-4B & 2560 & 9728 & 1 & 49.1 & 10.4 & $4.74\times$ / $3.21\times$\\
Qwen3-4B & 2560 & 9728 & 2048 & 539.3 & 193.8 & $2.78\times$ / $2.68\times$\\
Qwen3-8B & 4096 & 12288 & 1 & 83.6 & 12.6 & $6.64\times$ / $4.12\times$\\
Qwen3-8B & 4096 & 12288 & 2048 & 1031.9 & 247.3 & $4.17\times$ / $3.92\times$\\
Qwen3-30B-A3B expert & 2048 & 768 & 1 & 11.5 & 6.7 & $1.71\times$ / $1.42\times$\\
Qwen3-30B-A3B expert & 2048 & 768 & 2048 & 40.3 & 23.2 & $1.74\times$ / $1.54\times$\\
DeepSeek-V3 expert & 7168 & 2048 & 1 & 33.4 & 8.5 & $3.94\times$ / $2.75\times$\\
DeepSeek-V3 expert & 7168 & 2048 & 2048 & 304.2 & 81.3 & $3.74\times$ / $3.40\times$\\
\bottomrule
\end{tabular*}
\end{table*}
\begin{table*}[t]
\centering
\caption{Forward plus activation and weight gradients, $M=2048$, $q=32$, warm replay. The dense library preference is calibrated separately for each of its three GEMMs. Times are microseconds.}
\label{tab:kernel-training}
\small
\begin{tabular}{llrrr}
\toprule
Model shape & Direction & Dense & Algebra & Speedup\\
\midrule
Qwen3-1.7B & up & 225.6 & 177.0 & $1.27\times$\\
Qwen3-1.7B & down & 223.8 & 127.0 & $1.76\times$\\
Qwen3-4B & up & 455.9 & 341.6 & $1.33\times$\\
Qwen3-4B & down & 465.9 & 280.2 & $1.66\times$\\
Qwen3-8B & up & 912.0 & 465.0 & $1.96\times$\\
Qwen3-8B & down & 954.0 & 372.6 & $2.56\times$\\
Qwen3-30B-A3B expert & up & 40.0 & 61.4 & $0.65\times$\\
Qwen3-30B-A3B expert & down & 38.5 & 67.6 & $0.57\times$\\
DeepSeek-V3 expert & up & 260.8 & 146.1 & $1.79\times$\\
DeepSeek-V3 expert & down & 275.9 & 205.0 & $1.35\times$\\
\bottomrule
\end{tabular}
\end{table*}
\begin{table*}[t]
\centering
\caption{Pure $D\to F$ projection: optimized dense GEMM versus the algebraic product at fixed $q=32$, BF16 on H200 NVL. Times are warm-replay medians in microseconds. The last column gives warm / cache-flushed speedup. No activation or FFN fusion is included.}
\label{tab:kernel-products-fixed}
\small
\begin{tabular*}{\textwidth}{@{\extracolsep{\fill}}lrrrrrr@{}}
\toprule
Model shape & $D$ & $F$ & $M$ & Dense $\mu$s & Algebra $\mu$s & Speedup (warm / flushed)\\
\midrule
Qwen3-1.7B & 2048 & 6144 & 1 & 9.5 & 3.8 & $2.51\times$ / $1.89\times$\\
Qwen3-1.7B & 2048 & 6144 & 2048 & 75.7 & 24.9 & $3.04\times$ / $2.84\times$\\
Qwen3-4B & 2560 & 9728 & 1 & 14.0 & 4.1 & $3.43\times$ / $2.72\times$\\
Qwen3-4B & 2560 & 9728 & 2048 & 160.6 & 51.1 & $3.14\times$ / $2.94\times$\\
Qwen3-8B & 4096 & 12288 & 1 & 30.7 & 4.4 & $7.07\times$ / $4.14\times$\\
Qwen3-8B & 4096 & 12288 & 2048 & 320.7 & 63.6 & $5.04\times$ / $4.52\times$\\
Qwen3-30B-A3B expert & 2048 & 768 & 1 & 6.5 & 3.3 & $2.00\times$ / $1.36\times$\\
Qwen3-30B-A3B expert & 2048 & 768 & 2048 & 13.4 & 10.5 & $1.28\times$ / $1.14\times$\\
DeepSeek-V3 expert & 7168 & 2048 & 1 & 10.9 & 3.8 & $2.87\times$ / $2.20\times$\\
DeepSeek-V3 expert & 7168 & 2048 & 2048 & 82.4 & 28.5 & $2.89\times$ / $2.51\times$\\
\bottomrule
\end{tabular*}
\end{table*}

\section{Larger Pure Products}
\label{app:large-products}
The extended experiment keeps the pure-product comparison and synthetic BF16 inputs. It tests both Qwen3-32B FFN projection directions at $M\in\{2048,8192,32768\}$ and $q\in\{32,64,128,256\}$, and square products with $n\in\{4096,8192,16384,32768\}$ and fixed block sizes $b\in\{128,256\}$. The public Qwen3-32B configuration gives hidden width 5120 and intermediate width 25600. No model weights are loaded.

Calibration tests eight tile configurations for each of the fused and two-stage Q schedules and selects the dense cuBLAS/cuBLASLt preference. Each calibration uses three short timing trials; twelve separate evaluation trials alternate the two methods. Warm trials use up to 40 graph replays, choosing a repeat count that targets at least 2\,ms for the faster operation when possible. Cache-flushed trials use one replay after touching 512\,MiB outside the timed interval. Large tensors are allocated once per shape. Only the forward product is measured in this extension.

Every candidate is checked against independent FP32 sums at at least 2048 sampled output entries, covering diagonal and off-diagonal terms and the full row range. The largest sampled relative $L_2$ error is 0.00196. Additional $q=128,256$ checks exercise nonzero type offsets and row counts not divisible by $q$. The larger-size grid contains 32 configurations, each measured under both cache protocols.

At fixed $b$, each operand and the output contain $n^2$ entries, while the number of active block products is $2q^2-q$. The hardware measurements test the finite-size behavior of this family. They do not establish a quality-preserving model scaling law; $q$ changes between members of the family.

Across the tested eightfold increase in $n$, the $b=128$ Q latency has endpoint log--log slope 2.01. This is a description of these four measured sizes, separate from the exact quadratic arithmetic bound in \cref{prop:quadratic-scaling}.

\begin{table*}[t]
\centering
\caption{Qwen3-32B projection shapes at four subdivision sizes. Block dimensions $(m,k,n)=(M/q,K/q,N/q)$ precede GPU-tile padding. Times are warm medians in milliseconds; bold Q times are the lowest among the four tested $q$ values for that shape. The last column gives warm / cache-flushed speedup.}
\label{tab:large-projections}
\small
\begin{tabular}{llrcrrr}
\toprule
$K\to N$ & $M$ & $q$ & Block $(m,k,n)$ & Dense ms & Q ms & Speedup\\
\midrule
$5120\to25600$ & 2048 & 32 & $(64,160,800)$ & 0.849 & \textbf{0.168} & $5.06\times$ / $4.86\times$\\
 &  & 64 & $(32,80,400)$ & 0.829 & 0.232 & $3.58\times$ / $3.75\times$\\
 &  & 128 & $(16,40,200)$ & 0.779 & 0.248 & $3.15\times$ / $3.25\times$\\
 &  & 256 & $(8,20,100)$ & 0.789 & 0.266 & $2.97\times$ / $3.19\times$\\
\addlinespace
$5120\to25600$ & 8192 & 32 & $(256,160,800)$ & 3.290 & 0.583 & $5.65\times$ / $5.88\times$\\
 &  & 64 & $(128,80,400)$ & 3.290 & 0.523 & $6.30\times$ / $6.68\times$\\
 &  & 128 & $(64,40,200)$ & 3.337 & \textbf{0.401} & $8.33\times$ / $8.60\times$\\
 &  & 256 & $(32,20,100)$ & 3.316 & 0.414 & $8.00\times$ / $8.10\times$\\
\addlinespace
$5120\to25600$ & 32768 & 32 & $(1024,160,800)$ & 13.446 & 2.335 & $5.76\times$ / $5.76\times$\\
 &  & 64 & $(512,80,400)$ & 13.336 & 2.055 & $6.49\times$ / $6.83\times$\\
 &  & 128 & $(256,40,200)$ & 13.596 & 1.558 & $8.73\times$ / $8.68\times$\\
 &  & 256 & $(128,20,100)$ & 13.459 & \textbf{1.169} & $11.51\times$ / $11.70\times$\\
\addlinespace
$25600\to5120$ & 2048 & 32 & $(64,800,160)$ & 0.810 & 0.181 & $4.49\times$ / $4.27\times$\\
 &  & 64 & $(32,400,80)$ & 0.731 & \textbf{0.147} & $4.97\times$ / $4.85\times$\\
 &  & 128 & $(16,200,40)$ & 0.813 & 0.174 & $4.66\times$ / $4.41\times$\\
 &  & 256 & $(8,100,20)$ & 0.778 & 0.243 & $3.20\times$ / $3.32\times$\\
\addlinespace
$25600\to5120$ & 8192 & 32 & $(256,800,160)$ & 3.339 & 0.576 & $5.79\times$ / $5.95\times$\\
 &  & 64 & $(128,400,80)$ & 3.354 & \textbf{0.396} & $8.48\times$ / $8.30\times$\\
 &  & 128 & $(64,200,40)$ & 3.312 & 0.411 & $8.06\times$ / $8.07\times$\\
 &  & 256 & $(32,100,20)$ & 3.294 & 0.552 & $5.96\times$ / $6.38\times$\\
\addlinespace
$25600\to5120$ & 32768 & 32 & $(1024,800,160)$ & 13.564 & 2.354 & $5.76\times$ / $5.72\times$\\
 &  & 64 & $(512,400,80)$ & 13.509 & \textbf{1.519} & $8.89\times$ / $8.92\times$\\
 &  & 128 & $(256,200,40)$ & 13.356 & 1.684 & $7.93\times$ / $7.97\times$\\
 &  & 256 & $(128,100,20)$ & 13.248 & 1.958 & $6.76\times$ / $6.84\times$\\
\addlinespace
\bottomrule
\end{tabular}
\end{table*}
\begin{table*}[t]
\centering
\caption{Fixed-block-size square products. Subdivision stops at $b\times b$ blocks; algebra size grows as $q=n/b$. Warm medians are milliseconds; bold Q times are the lower of the two tested block sizes for that $n$. The final column includes the separate cache-flushed measurement.}
\label{tab:large-square}
\small
\begin{tabular}{rrrrrr}
\toprule
$n$ & $b$ & $q$ & Dense ms & Q ms & Speedup (warm / flushed)\\
\midrule
4096 & 128 & 32 & 0.210 & \textbf{0.044} & $4.81\times$ / $4.47\times$\\
4096 & 256 & 16 & 0.205 & 0.058 & $3.56\times$ / $3.19\times$\\
8192 & 128 & 64 & 1.773 & \textbf{0.166} & $10.69\times$ / $10.63\times$\\
8192 & 256 & 32 & 1.729 & 0.225 & $7.70\times$ / $7.42\times$\\
16384 & 128 & 128 & 13.881 & \textbf{0.689} & $20.15\times$ / $19.95\times$\\
16384 & 256 & 64 & 13.999 & 0.991 & $14.13\times$ / $14.08\times$\\
32768 & 128 & 256 & 114.265 & \textbf{2.868} & $39.85\times$ / $38.85\times$\\
32768 & 256 & 128 & 117.518 & 5.344 & $21.99\times$ / $21.58\times$\\
\bottomrule
\end{tabular}
\end{table*}

\section{Attention Extension and Detailed Cost Accounting}
\label{sec:attention-details}

An autoregressive Transformer uses the same trained layer in two computational regimes.  Teacher-forced training evaluates all token rows and all allowed query--key pairs in parallel under the triangular mask.  Inference first processes the prompt in parallel and then contributes one new row per active sequence while reusing the KV cache.  The algebra, weights, and position types remain fixed across both regimes.

\begin{algorithm}[t]
\caption{Parallel training with algebraic attention.}
\label{alg:attention-training}
\begin{algorithmic}[1]
\REQUIRE Hidden rows $H_b$, absolute positions $p_b$, weight banks
\STATE $i_b\gets\tau(p_b)$ for every row $b$
\STATE Compute $Q,K,V$ with Algorithm~\ref{alg:projection}
\FOR{each allowed query--key pair $(b,u)$}
  \STATE Compute $s_{bu}$ using~\eqref{eq:attention-raw-score}--\eqref{eq:attention-score}
\ENDFOR
\STATE $P_{bu}\gets\softmax_{u:\,p_u\le p_b}(s_{bu})$
\STATE $O_b\gets\sum_{u:\,p_u\le p_b}P_{bu}V_u$
\STATE Apply the typed output and feed-forward projections
\STATE Compute next-token loss and backpropagate
\end{algorithmic}
\end{algorithm}

\begin{algorithm}[t]
\caption{Cached decoding with algebraic attention.}
\label{alg:attention-decoding}
\begin{algorithmic}[1]
\REQUIRE New row $h_p$, full cache $(K_u,V_u)_{u<p}$, weight banks
\STATE $i\gets\tau(p)$
\STATE Compute $Q_p,K_p,V_p$ with Algorithm~\ref{alg:projection}
\STATE Append full $K_p,V_p$ to the cache
\FOR{each allowed cached position $u\le p$}
  \STATE $j\gets\tau(u)$
  \STATE Compute $s_{pu}$ using~\eqref{eq:attention-raw-score}--\eqref{eq:attention-score}
\ENDFOR
\STATE $P_p\gets\softmax(s_p)$; $O_p\gets\sum_{u\le p}P_{pu}V_u$
\STATE Apply the typed output and feed-forward projections
\STATE \textbf{return} updated hidden row and KV cache
\end{algorithmic}
\end{algorithm}

The two algorithms use three properties of the construction.  Each projection acts on one token row; the type $\tau(p)$ is determined by the position; and a pair score depends only on the current query, one key, and their two types.  Hence a decoder computes each $K_u,V_u$ pair once and reuses it for every future query.  The triangular mask determines the temporal dependency, while the algebra selects feature groups inside each admitted pair.

\subsection{Training}

For one attention head, let $d_h=qm$ and split $Q_p,K_u\in\R^{d_h}$ into $q$ groups of width $m$.  For types $i=\tau(p)$ and $j=\tau(u)$, define
\begin{equation}
F(i,j)=
\begin{cases}
\{i\},&i=j,\\
\{i,j\},&i\neq j,
\end{cases}
\label{eq:feature-set}
\end{equation}
The algebra determines the unnormalized score
\begin{equation}
r_{pu}=
\sum_{a\in F(i,j)}
\left\langle Q_p^{(a)},K_u^{(a)}\right\rangle .
\label{eq:attention-raw-score}
\end{equation}
Its support contains
\begin{equation}
k_{ij}=|F(i,j)|m
\label{eq:attention-support-size}
\end{equation}
scalar coordinate products: $m$ when $i=j$ and $2m$ otherwise.  To derive a default scale, suppose at initialization that the selected query and key coordinates are independent, centered, and have unit variance.  Then $\operatorname{Var}(r_{pu})=k_{ij}$.  The variance-matched score is therefore
\begin{equation}
s_{pu}=\frac{r_{pu}}{\sqrt{k_{ij}}}.
\label{eq:attention-score}
\end{equation}
Attention applies this scaling to logits before softmax.  In ordinary dense attention the score uses all $d_h$ coordinates, so $k_{ij}=d_h$ and~\eqref{eq:attention-score} reduces to the standard $1/\sqrt{d_h}$ rule.  The algebraic score uses only $m$ or $2m$ coordinates, giving $1/\sqrt m$ or $1/\sqrt{2m}$.  Before normalization,~\eqref{eq:attention-raw-score} is precisely the corresponding block of $Q\star_qK^\top$:
\[
S_{ii}=Q_{ii}K_{ii}^{\top},\qquad
S_{ij}=Q_{ii}K_{ji}^{\top}+Q_{ij}K_{jj}^{\top}\quad(i\neq j).
\]
For $\cQ_2^{\otimes2}$, types are bit strings $\alpha,\delta\in\{0,1\}^2$ and the selected groups are
\begin{equation}
F_{\otimes}(\alpha,\delta)
=\{\beta:\beta_k\in\{\alpha_k,\delta_k\},\ k=1,2\}.
\label{eq:tensor-square-feature-set}
\end{equation}
A pair at Hamming distance $h(\alpha,\delta)$ uses $2^h$ groups, hence $k_{\alpha\delta}=2^hm$ and scale $1/\sqrt{2^hm}$.

We use the support-size normalization in~\eqref{eq:attention-score}.  Algebra groups lie within each attention head, and every rotary-embedding pair stays within one group.

The temporal mask and value aggregation are
\begin{align}
\widetilde s_{pu}&=
\begin{cases}s_{pu},&u\le p,\\-\infty,&u>p,\end{cases}\label{eq:causal-mask}\\
P_{pu}&=\softmax_{u}(\widetilde s_{pu}),\qquad
O_p=\sum_{u\le p}P_{pu}V_u.
\label{eq:softmax-pv}
\end{align}
The first design uses the algebraic score together with dense $PV$.

\begin{proposition}[Causality]
Suppose the hidden state at position $p$ depends only on input positions at most $p$.  Algebraic projections, the score in~\eqref{eq:attention-score}, and the masked aggregation in~\eqref{eq:causal-mask}--\eqref{eq:softmax-pv} preserve this property.
\label{prop:attention-causality}
\end{proposition}

\begin{proof}
Equations~\eqref{eq:typed-diag}--\eqref{eq:typed-off} act independently on each token row.  At query position $p$, the mask admits keys and values at $u\le p$.  The product selects feature groups within an admitted pair and therefore preserves the temporal dependency graph.
\end{proof}

For projections, $dX$ and $dW$ repeat the forward block support, giving the three-pass count in~\eqref{eq:three-pass-projection-cost}.  For equal groups the resulting projection fractions are $\rho_2=3/4$, $\rho_4=7/16$, and $9/16$ for the recursive $\cQ_2^{\otimes2}$ construction.

For $\cQ_q$ with equal group widths, let $M$ be the number of allowed query--key pairs and $M_{=}$ the number with equal endpoint types.  The exact score fraction is
\begin{equation}
\rho_{\rm score}=\frac{2M-M_{=}}{qM}.
\label{eq:attention-exact-ratio}
\end{equation}
For balanced long sequences, $\rho_{\rm score}\to\rho_q$.  Counting only the attention contractions, structured $QK^\top$ with dense $PV$ has the arithmetic speedup ceilings
\begin{equation}
S_{\rm attn,fwd}=\frac{2}{1+\rho_{\rm score}},\qquad
S_{\rm attn,train}=\frac{7}{3+4\rho_{\rm score}},
\label{eq:attention-ceilings}
\end{equation}
where the training count uses score recomputation in the backward pass and includes $dQ$ and $dK$.
Forward plus backward then contains four structured score contractions---forward $QK^\top$, score recomputation, $dQ$, and $dK$---and three dense value contractions---$PV$, $dP$, and $dV$.

\begin{table}[t]
\centering
\caption{Balanced-sequence arithmetic ceilings for structured $QK^\top$ and dense $PV$.}
\label{tab:attention-ceilings}
\small
\begin{tabular}{lccc}
\toprule
Law & $\rho$ & Forward & Training\\
\midrule
$\cQ_2^{\otimes2}$ & $9/16$ & $1.280\times$ & $1.333\times$\\
$\cQ_4$ & $7/16$ & $1.391\times$ & $1.474\times$\\
\bottomrule
\end{tabular}
\end{table}

Let $\mathcal P$ be the learned projections in one decoder block.  If projection $\ell$ sees $B_{\ell i}$ rows of type $i$ and $B_\ell=\sum_iB_{\ell i}$, set
\begin{align}
C_{\rm proj,dense}^{\rm fwd}
&=\sum_{\ell\in\mathcal P}B_\ell d_\ell^{\rm in}d_\ell^{\rm out},\nonumber\\
C_{\rm proj,alg}^{\rm fwd}
&=\sum_{\ell\in\mathcal P}\left[
B_\ell D_{{\rm diag},\ell}
+\sum_iB_{\ell i}s_{\ell i}(d_\ell^{\rm out}-r_{\ell i})\right].
\label{eq:projection-list-cost}
\end{align}
Under the standard MHA/GQA convention with aggregate query width $H_q=n_qd_h$ and $M$ allowed attention pairs, the projection and attention-contraction forward MAC counts are
\begin{align}
C_{\rm dense}^{\rm fwd}&=C_{\rm proj,dense}^{\rm fwd}+2MH_q,\label{eq:dense-block-cost}\\
C_{\rm alg}^{\rm fwd}&=C_{\rm proj,alg}^{\rm fwd}+(1+\rho_{\rm score})MH_q,
\label{eq:alg-block-cost}
\end{align}
and forward plus backward gives
\begin{align}
C_{\rm dense}^{\rm train}&\simeq3C_{\rm proj,dense}^{\rm fwd}+7MH_q,\label{eq:dense-train-cost}\\
C_{\rm alg}^{\rm train}&\simeq3C_{\rm proj,alg}^{\rm fwd}+(3+4\rho_{\rm score})MH_q.
\label{eq:alg-train-cost}
\end{align}
GQA changes the $K,V$ projection widths and cache storage through $H_{kv}$; every query head still performs its own score and value contraction, which gives the $H_q$ factor above.

\subsection{Inference}

Prompt prefill uses the parallel computation above and initializes the cache.  At each layer and for each sequence, the cache after position $p$ contains
\begin{equation}
K^{\rm cache},V^{\rm cache}
\in\R^{N_p\times n_{kv}\times d_h},
\label{eq:kv-cache-shape}
\end{equation}
where $N_p$ is the number of retained positions.  Thus the storage is $2N_pH_{kv}$ scalars, exactly as in the dense Transformer.  Keys and values keep all $d_h$ coordinates; the type of entry $u$ is recovered as $\tau(u)$ from its position.

Algorithm~\ref{alg:attention-decoding} gives the one-row update. Under GQA, each query head uses its assigned key/value head.  Each cached key and value is computed once and reused by every later position.  Different future query types may select different groups of the same cached key, which is why the cache stores the complete key vector.  Past queries and hidden rows do not enter the update.

For example, with $q=4$, a type-$3$ query reads group $3$ from a cached type-$3$ key and groups $\{3,1\}$ from a cached type-$1$ key.  A later type-$2$ query reads groups $\{2,1\}$ from that same type-$1$ key.  The full cached key supports both queries without recomputation.

Writing $D_{{\rm diag},\ell}=\sum_a s_{\ell,a}r_{\ell,a}$, applying~\eqref{eq:unequal-row-cost} to every projection $\ell\in\mathcal P$ gives
\begin{align}
C_{\rm proj,alg}^{(i),\rm dec}
&=\sum_{\ell\in\mathcal P}
\left[D_{{\rm diag},\ell}
+s_{\ell,i}\bigl(d_\ell^{\rm out}-r_{\ell,i}\bigr)\right],\\
C_{\rm proj,dense}^{\rm dec}
&=\sum_{\ell\in\mathcal P}d_\ell^{\rm in}d_\ell^{\rm out}.
\label{eq:decode-projection-total}
\end{align}
Unequal group widths make the algebraic cost position-type dependent.

For a standard decoder block, let $D$ be the model width, $F$ the MLP width, and $g=1$ for a two-projection MLP or $g=2$ for a gated MLP with two $D\to F$ maps.  The dense one-token projection cost is
\begin{equation}
C_{\rm proj,dense}^{\rm dec}
=2DH_q+2DH_{kv}+(g+1)DF.
\label{eq:decode-projection-expanded}
\end{equation}
With equal algebra groups and all these projections replaced, every token type has the same cost
\begin{equation}
C_{\rm proj,alg}^{\rm dec}=\rho_q C_{\rm proj,dense}^{\rm dec},
\qquad \rho_q=\frac{2q-1}{q^2}.
\label{eq:decode-projection-equal}
\end{equation}

After appending $K_p$ and $V_p$, let the cache contain $N_p$ allowed keys, with $n_j$ keys of type $j$ and $\sum_jn_j=N_p$.  A type-$i$ query head uses one group for its $n_i$ matching keys and two groups for the others.  Its exact score cost is
\begin{equation}
C_{QK,{\rm head}}^{(i)}=m(2N_p-n_i).
\label{eq:decode-score-one-head}
\end{equation}
Across $n_q$ query heads,
\begin{equation}
C_{QK,{\rm alg}}^{(i)}=\frac{H_q}{q}(2N_p-n_i),
\qquad
\rho_{{\rm score},{\rm dec}}^{(i)}=\frac{2N_p-n_i}{qN_p}.
\label{eq:decode-score-ratio}
\end{equation}
Dense value aggregation costs $N_pH_q$.  The one-token projection and attention-contraction MAC counts are therefore
\begin{align}
C_{\rm dense}^{(i),\rm dec}
&=C_{\rm proj,dense}^{\rm dec}+2N_pH_q,\\
C_{\rm alg}^{(i),\rm dec}
&=C_{\rm proj,alg}^{(i),\rm dec}+
\left(1+\rho_{{\rm score},{\rm dec}}^{(i)}\right)N_pH_q.
\label{eq:decode-block-total}
\end{align}

\begin{table}[t]
\centering
\caption{One-token decode at one layer.  Projection entries assume equal algebra groups; attention entries hold for the current cache composition.}
\label{tab:decode-summary}
\small
\begin{tabular}{lcc}
\toprule
Quantity & Dense & $\cQ_q$\\
\midrule
KV cache scalars & $2N_pH_{kv}$ & $2N_pH_{kv}$\\
Projection MACs & $C_{\rm proj,dense}^{\rm dec}$ & $\rho_qC_{\rm proj,dense}^{\rm dec}$\\
Score MACs & $N_pH_q$ & $\rho_{{\rm score},{\rm dec}}^{(i)}N_pH_q$\\
Value MACs & $N_pH_q$ & $N_pH_q$\\
\bottomrule
\end{tabular}
\end{table}

For a balanced cache, $n_i=N_p/q$ and $\rho_{{\rm score},{\rm dec}}^{(i)}=\rho_q$.  For example, $\cQ_4$ reduces the score contraction to $7N_pH_q/16$.  Dense value aggregation still costs $N_pH_q$, so attention as a whole costs $23/32$ of dense attention, an arithmetic speedup of $32/23\approx1.39$.  Decode time remains linear in $N_pH_q$ and cache storage remains linear in $N_pH_{kv}$; the changed product improves the constants in the projections and score contraction, while dense value aggregation sets the attention floor.

For $\cQ_2^{\otimes2}$, a query type $\alpha$ has exact per-head score cost
\begin{equation}
C_{QK,{\rm head}}^{(\alpha)}
=m\sum_{\delta\in\{0,1\}^2}2^{h(\alpha,\delta)}n_\delta.
\label{eq:decode-score-tensor-square}
\end{equation}
A balanced cache gives score fraction $9/16$ and attention-only speedup $1.28$.

Decode exposes skinny GEMMs or GEMVs and can be limited by weight and cache bandwidth.  Under GQA, the cache stores $2N_pH_{kv}$ key/value scalars, while the arithmetic scales with $H_q$ because every query head has its own scores and probabilities.  A lower arithmetic count therefore need not produce the same relative speedup as in training.  Timing this extension requires per-token latency and throughput measurements across batch sizes and cache lengths.

\section{Graph-Algebra and Scaling Proofs}
\label{app:graph-proofs}

\subsection{Graph Algebra}

\begin{proof}[Proof of \cref{prop:structure}]
A triple of basis elements with at least two edges vanishes under either
bracketing. With one edge, both bracketings retain that edge exactly when
the adjacent vertices match its source and target. With no edges, both
require matching vertices. Bilinearity proves associativity. The sum
$\sum_i e_i$ acts as an identity on every basis element. Multiplying an
edge from either side returns an edge or zero, so its span is a two-sided
ideal and $J^2=0$. Removing that span leaves the independent products
$e_i e_i=e_i$, giving $\cQ_G/J\cong\F^{|V|}$.
\end{proof}

\begin{proof}[Proof of \cref{thm:graph-rank}]
Equations~\eqref{eq:graph-product-vertex}--\eqref{eq:graph-product-edge}
give the upper bound $R(\cQ_G)\le |V|+2|E|$. Every maximal two-sided ideal
contains $J$: its image in the semisimple quotient cannot contain a
nonzero nilpotent ideal. Maximal ideals therefore correspond to the
$|V|$ maximal ideals of $\F^{|V|}$, so $t(\cQ_G)=|V|$. The
Alder--Strassen bound then gives
\[
R(\cQ_G)\ge2(|V|+|E|)-|V|=|V|+2|E|,
\]
which matches the upper bound.
\end{proof}

\subsection{Quadratic Scaling}

\label{app:scaling-proof}
\begin{proof}[Proof of \cref{prop:quadratic-scaling}]
There are $2q^2-q$ ordinary block products, each costing $b^3$ MACs.
Since $n=qb$,
\[
(2q^2-q)b^3
=\left(\frac2q-\frac1{q^2}\right)n^3
=2bn^2-b^2n.
\]
The edge ideal tensored with $M_b(\F)$ is nilpotent, while the quotient is
a direct product of $q$ copies of $M_b(\F)$. Hence $t(\mathcal A_{q,b})=q$,
and Alder--Strassen gives
\[
R(\mathcal A_{q,b})
\ge2n^2-q.
\]
\end{proof}

\subsection{Causality}
\label{app:causality-proof}

\begin{proof}[Proof of \cref{prop:causality}]
The row action depends only on the features and type at the current
position. At position $p$, triangular attention admits keys and values only
at positions $u\le p$. Induction over layers therefore preserves dependence
only on positions at most $p$.
\end{proof}

\clearpage
\onecolumn
\section{Training Data Sources}
\label{app:training_dataset}

The training and evaluation data comprise a mixture of open-source datasets available on Hugging Face~\footnote{\url{https://huggingface.co/datasets}}.
\autoref{tab:ds_name_stats} reports the number of records imported from each source dataset.
The training split includes only the training splits of the original datasets.

\paragraph{Deduplication.}

To avoid unintended duplication from source datasets that combine other
datasets, we retain only the first occurrence of each \texttt{query}--\texttt{response}
pair, where duplicates are defined by an exact character-by-character match.

\paragraph{Decontamination and junk filtering.}
For each raw sample, the \texttt{query} field is checked for near-duplicates against a pool of benchmark queries from MBPP, HumanEval, MMLU, GSM8K, HellaSwag, and PIQA. Similarity is measured using Jaccard similarity over 3-character shingles, with MinHash + LSH used for efficient candidate generation followed by an exact Jaccard recheck. For samples with \texttt{domain == "coding"}, every \texttt{```python ... ```} fenced code snippet in the \texttt{response} field is additionally compared against MBPP reference solutions in the \texttt{code} field. This detects cases in which the query has been paraphrased or rewritten but the response remains close to a benchmark solution. Samples with Jaccard similarity above \texttt{DECONTAM\_JACCARD\_THRESHOLD=0.80} are removed.

Trivially empty, too-short, and degenerate samples are also removed using domain-specific checks: code-specific checks for \texttt{domain == "coding"} and a generic repetition check otherwise.

\small
\begin{longtable}{@{}p{0.80\textwidth}r@{}}
\caption{Imported records by dataset, as recorded for the training mixture.}
\label{tab:ds_name_stats}\\
\toprule
Dataset & Records\\
\midrule
\endfirsthead
\toprule
Dataset (continued) & Records\\
\midrule
\endhead
\bottomrule
\endfoot
GSAI-ML/ReFusion & 2881860 \\
openbmb/UltraData-SFT-2605 & 2651052 \\
MBZUAI/LaMini-instruction & 2540109 \\
OLMo-Coding/starcoder-python-instruct & 1087321 \\
nvidia/Nemotron-Post-Training-Dataset-v2 & 1145725 \\
nvidia/Nemotron-RL-knowledge-mcqa & 611930 \\
nvidia/Nemotron-SFT-Instruction-Following-Chat-v2 & 604155 \\
ZeroAgency/ru-big-russian-dataset & 441537 \\
t-tech/T-Wix & 438000 \\
meta-math/MetaMathQA & 381136 \\
OpenCoder-LLM/opc-sft-stage2 & 280319 \\
jtatman/python-code-dataset-500k & 207994 \\
NTU-NLP-sg/xCodeEval & 204545 \\
d0rj/orca-math-word-problems-200k-ru & 197488 \\
Helsinki-NLP/opus-100 & 165138 \\
sayhan/strix-philosophy-qa & 133760 \\
open-r1/OpenThoughts-114k-math & 88843 \\
allenai/tulu-3-sft-personas-math-filtered & 80535 \\
openbmb/UltraInteract\_sft & 75781 \\
Post-training-Data-Flywheel/AutoIF-instruct-61k & 61345 \\
argilla/ifeval-like-data & 56054 \\
camel-ai/math & 49079 \\
MERA-evaluation/MERA & 48771 \\
bingbangboom/philosophia-QA & 47394 \\
allenai/tulu-3-sft-personas-math-grade-filtered & 47042 \\
MuskumPillerum/General-Knowledge & 34877 \\
Vikhrmodels/GrandMaster-PRO-MAX & 29123 \\
AITISPEC/physics-russian & 19780 \\
attn-signs/russian-code & 19148 \\
MexIvanov/CodeExercise-Python-27k-ru & 18923 \\
ajibawa-2023/Python-Code-23k-ShareGPT & 17092 \\
teknium/OpenHermes-2.5 & 14649 \\
RushabhShah122000/python-expert-dataset & 12567 \\
qwedsacf/competition\_math & 11613 \\
microsoft/NextCoderDataset & 9319 \\
mizinovmv/ru\_ifeval-like-data & 3822 \\
attn-signs/russian-easy-instructions & 2101 \\
newfacade/LeetCodeDataset & 2082 \\
greengerong/leetcode & 2050 \\
jondurbin/airoboros-3.2 & 1710 \\
ise-uiuc/Magicoder-Evol-Instruct-110K & 656 \\
nvidia/Nemotron-SFT-ARC-AGI-v1 & 521 \\
LLiserginov/russian-instructions-10k & 178 \\
HuggingFaceH4/ifeval-like-data & 3 \\
\end{longtable}
\normalsize

\end{document}